\documentclass[aoas,preprint]{imsart}

\RequirePackage[OT1]{fontenc}
\RequirePackage{amsthm,amsmath}
\RequirePackage{natbib}
\RequirePackage[colorlinks,citecolor=blue,urlcolor=blue]{hyperref}

\usepackage{hyperref}
\usepackage{url}
\usepackage{wrapfig}

\usepackage{psfrag}
\usepackage{color}
\usepackage[ruled,vlined,linesnumbered]{algorithm2e}
\usepackage{algorithmic}
\usepackage{amsfonts}
\usepackage{amsmath}
\usepackage{amsthm}
\usepackage{mathrsfs}
\usepackage{amssymb}
\usepackage{wrapfig}
\usepackage{graphicx}
\usepackage{psfrag,subfigure}

\newcommand{\bbeta}{\boldsymbol{\beta}}

\newcommand{\etab}{\boldsymbol{\eta}}
\newcommand{\thetab}{\boldsymbol{\theta}}
\newcommand{\Thetab}{\boldsymbol{\Theta}}

\newcommand{\by}{\mathbf{y}}
\newcommand{\bX}{\mathbf{X}}
\newcommand{\bD}{\mathbf{D}}

\newcommand{\bx}{\mathbf{x}}
\newcommand{\br}{\mathbf{r}}

\newcommand{\bF}{\mathbf{F}}
\newcommand{\bg}{\mathbf{g}}
\newcommand{\bh}{\mathbf{h}}

\newcommand{\bo}{\mathbf{o}}

\newcommand{\bz}{\mathbf{z}}
\newcommand{\bv}{\mathbf{v}}
\newcommand{\bu}{\mathbf{u}}
\newcommand{\bw}{\mathbf{w}}
\newcommand{\bs}{\mathbf{s}}

\newcommand{\ba}{\mathbf{a}}
\def\bSig\mathbf{\Sigma}

\newcommand{\argmin}{\operatornamewithlimits{argmin}}
\newcommand\Ldots{\mathop{\lower.01ex\hbox{$\hdots$}}}

\newtheorem{theorem}{Theorem}

\usepackage{hyperref}
\usepackage{color}
\def\bSig\mathbf{\Sigma}

\startlocaldefs
\numberwithin{equation}{section}
\theoremstyle{plain}

\endlocaldefs

\begin{document}

\begin{frontmatter}
\title{Screening Rules  for Overlapping Group Lasso}
\runtitle{Screening Rules  for Overlapping Group Lasso}

\begin{aug}
\author{\fnms{Seunghak} \snm{Lee}} 
\and
\author{\fnms{Eric P.} \snm{Xing}}

\affiliation{Carnegie Mellon University}

\end{aug}

\begin{abstract}
Recently, to solve large-scale lasso and group lasso problems,
screening rules have been developed, the goal of which is
to reduce the problem size by efficiently discarding
zero coefficients using simple rules independently of the others.
However, 
screening for overlapping group lasso
remains an open challenge because
the  overlaps between groups
make it infeasible to test each group independently. 
In this paper, we develop screening rules for overlapping 
group lasso.
To address the challenge arising from  groups with overlaps, 
we
take into account  
overlapping groups only if they are 
inclusive of the group being tested, and
then we
derive 
screening rules, adopting the dual polytope projection approach.
This  strategy allows us to
screen each group independently of each other. 
In our experiments, 
we demonstrate the efficiency of our screening rules on  
various datasets.
\end{abstract}

\end{frontmatter}

\section{Introduction}
We propose efficient screening 
rules for  regression with the overlapping group lasso penalty.
Our goal is to develop simple rules to discard groups with zero coefficients in 
the optimization problem with the following form: 
\begin{equation}
\label{equ:general}
\min_{\bbeta} \frac{1}{2}
\lVert \by - \bX\bbeta  \rVert_2^2 
+ \lambda \sum_{\bg\in \mathcal{G}} \sqrt{n_\bg} \left\| \bbeta_{\bg} \right\|_2, 
\end{equation}
where $\bX \in \mathbb{R}^{N \times J}$ is the input data for  $J$ inputs and $N$ samples, 
$\by \in \mathbb{R}^{N \times 1}$ is the output vector, 
$\bbeta \in \mathbb{R}^{J \times 1}$ is the vector of regression coefficients, 
$n_\bg$ is the size of group $\bg$, and 
$\lambda$ is a regularization parameter that determines the sparsity of $\bbeta$.
In this setting, $\mathcal{G}$ represents a
set of groups of coefficients, defined {\it a priori}, and we allow 
arbitrary overlap between different groups, hence ``overlapping'' group lasso.
Overlapping group lasso is a general model that subsumes lasso \cite[]{tibshirani1996regression}, 
group lasso \cite[]{yuan2006model}, sparse group lasso \cite[]{simon2013sparse},
composite absolute penalties \cite[]{zhao2009composite}, and tree lasso \cite[]{zhao2009composite,kim2012tree} 
with $\ell_1/\ell_2$ penalty because they are a specific form of overlapping group lasso.

In this paper, we do not consider the  latent group lasso 
proposed by 
Jacob et al. \cite[]{jacob2009group}, 
where support is defined by the union of groups with nonzero coefficients. 
Instead, we consider the overlapping group lasso in the formulation of \eqref{equ:general}, 
where support is defined by the complement of the union of groups with zero coefficients
\cite[]{jenatton2011structured,yuan2011efficient}. 
Therefore, 
unlike the  latent group lasso,
simple conversion from \eqref{equ:general} to nonoverlapping group lasso problems
by duplicating  features overlapped between different groups
is infeasible.

Recently, to  solve \eqref{equ:general} efficiently, fast algorithms  have been developed 
\cite[]{yuan2011efficient,chen2012smoothing,deng2013group}
(we refer readers to \cite[]{bach2012optimization} for review of 
optimization with sparsity-inducing penalties); 
however, in many applications such as genome-wide association 
studies \cite[]{lee2012leveraging,yang2010identifying},
the number of features (or number of groups) can be very large. 
In such cases, fast optimization of \eqref{equ:general} is 
challenging because it requires us to sweep over all coefficients/groups of coefficients many times until the objective converges.
Furthermore, parallelization of existing sequential algorithms for speedup
is nontrivial.

The past years have seen the emergence of 
screening techniques that can
discard zero coefficients using simple rules in a single sweep 
over all coefficients.
Examples include Sasvi rules \cite[]{liu2013safe},
dual polytope projection (DPP) rules \cite[]{wang2013lasso}, 
dome tests \cite[]{xiang2012fast}, sphere tests \cite[]{xiang2011learning}, SAFE rules \cite[]{safefule2012}, 
and strong rules \cite[]{tibshirani2012strong}.
Among these, strong rules are not exact (true nonzero coefficients can be mistakenly discarded),
whereas the other tests are exact. 
Furthermore, 
Bonnefoy et al. \cite[]{bonnefoy2014} recently developed an approach that merges screening approach with first-order 
optimization algorithms for lasso;
however, to the best of our knowledge, none of the existing screening methods can be applied
to \eqref{equ:general} with overlapping groups. 
DPP and strong rules can be used for nonoverlapping group lasso, 
and the others are developed only for lasso.

\begin{figure} 
\begin{center}
\includegraphics[height=1.3in]{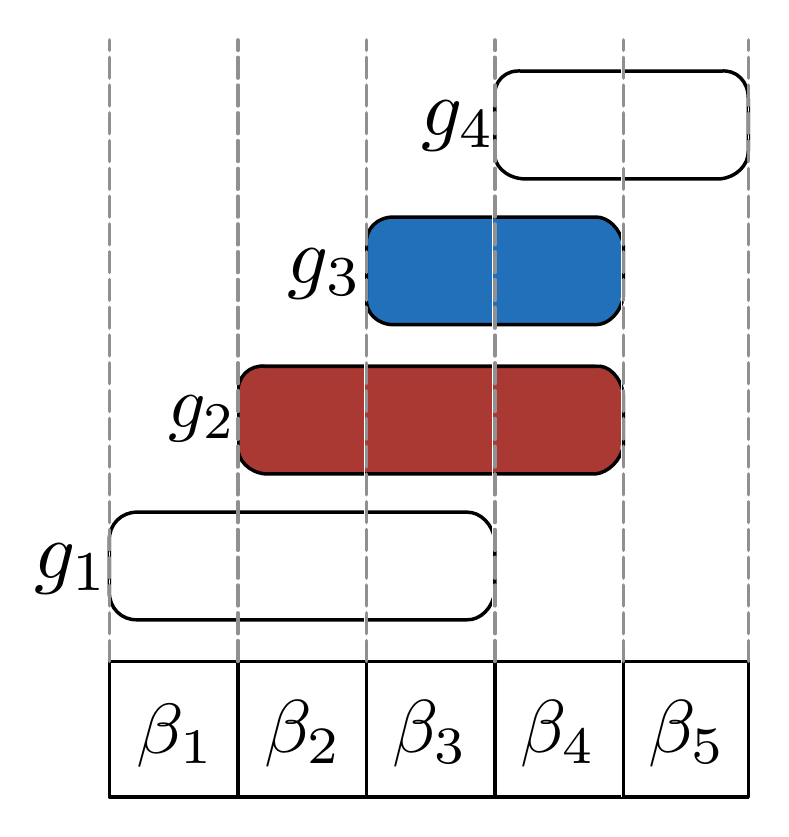}
\end{center}
\vspace{-0.4cm} 
\caption{\footnotesize 
For screening test on group $g_2$, we consider $g_3$ but disregard $g_1$ and $g_4$
to enable the independent screening test.}
\label{fig:ogroup_concept}
\vspace{-.4cm}
\end{figure}

In this paper, we develop exact screening rules for overlapping group lasso. 
The proposed screening rules can efficiently discard groups with zero coefficients 
by looking at each group independently. 
As a result, after screening, the number of groups that
potentially include nonzero coefficients is small, and 
thus we can reduce the size of  \eqref{equ:general}
by reformulating it
using only the groups that survived.
We then employ an optimization technique to solve the reduced problem.
The resultant solution is optimal because
the screening rules are exact in the sense that 
nonzero coefficients in a global optimal solution
are never mistakenly discarded.
The key idea behind our approach is to consider only groups that are inclusive of 
tested groups, while ignoring the other overlapping groups
to perform independent screening tests. For example, in Figure \ref{fig:ogroup_concept},
when performing a screening test on group $g_2$, we take into account only $g_3$ because $g_3$
is a subset of $g_2$, whereas $g_1$ and $g_4$ are non-inclusive of $g_2$.
The contributions of this paper are as follows:
\begin{enumerate}
\item We develop novel 
overlapping group lasso screening (OLS) 
and sparse overlapping group lasso screening (SOLS) rules.
Sparse overlapping group lasso is a special case of 
overlapping group lasso, as formulated in \eqref{equ:sparse_og_lasso}.
\item We show that 
the screening rule for nonoverlapping group lasso via DPP \cite[]{wang2013lasso}
(group DPP  or GDPP) 
is also  exact when applied to overlapping group lasso.
\end{enumerate}
In our experiments, we demonstrate that OLS, SOLS, and GDPP
give us significant speed-up against a solver without screening.
For example, OLS and SOLS achieved a $3.7\times$ and $3\times$ speed-up
on PIE image dataset,
compared to
an overlapping group lasso solver without screening. 
Furthermore,  OLS and SOLS are substantially more efficient to discard 
features with zero coefficients than 
GDPP under various experimental settings, confirming that 
the proposed algorithms are capable of using overlapping groups for screening.

\paragraph{Notation}
We refer matrices to boldface and uppercase letters;
vectors to boldface and lowercase letters; and
scalars to regular lowercase letters.
Columns are indexed by subscripts (e.g., $\bx_j$ is the $j$-th 
column vector of the matrix $\bX$).
We refer $\bg$ or $\bh$ to a group of coefficients, and
 $\bw_\bg$ represents a sub-vector of $\bw$, indexed by
$\bg$.

\section{Background: Screening Rules via DPP}

Recently, Wang et al. proposed screening rules via DPP
for nonoverlapping group lasso \cite[]{wang2013lasso}.
The DPP screening rules  are derived as follows: 
first, we find a dual form of group lasso and
its Karush-Kuhn-Tucker (KKT) conditions. 
Then, using the KKT conditions and the relationship between primal and dual 
solutions, we find screening rules to discard groups with zero coefficients;
however, such screening rules involve a dual optimal solution,  
which is unknown.
Thus,  the DPP approach finds a range of vectors
that includes a dual optimal solution, which is easy to obtain, 
and uses it instead of a dual optimal solution for screening rules.
Here we review the DPP  screening rule for nonoverlapping group lasso because it gives us with a vehicle to 
derive screening rules for overlapping group lasso. 

For the nonoverlapping group lasso, screening rules via DPP can be obtained by the following procedure:
\begin{enumerate}
\item Find KKT conditions for a dual  of the nonoverlapping group lasso.
\item Find the range of a dual optimal for the nonoverlapping
group lasso.
\item Derive screening rules by using the range of a dual optimal 
and the KKT conditions.
\end{enumerate}

The primal of the nonoverlapping group lasso is defined by
$$\min_{\bbeta} \frac{1}{2}\left\| \by - \bX \bbeta \right\|_2^2 + \lambda \sum_{\bg \in \mathcal{G}'} 
 \sqrt{n_\bg} \left\| \bbeta_{\bg} \right\|_2,$$ 
 where $\mathcal{G}'$ is a set of nonoverlapping groups,
 and $J = \sum_g n_\bg$.
We first represent the nonoverlapping group lasso in a dual form:
\begin{equation}
\sup_{\thetab} \left\{ \frac{1}{2} \left\| \by \right\|_2^2 -\frac{\lambda^2}{2} 
\left\| \thetab - \frac{\by}{\lambda} \right\|_2^2:  
\left\| \bX_{\bg}^T \thetab \right\|_2 \leq  \sqrt{n_\bg}, \, \forall 
\bg \in \mathcal{G}' \right\},
\label{eq:dual_gr_lasso}
\end{equation}
where $\thetab \in \mathbb{R}^{N \times 1}$ is a vector of dual variables.
In \eqref{eq:dual_gr_lasso}, one can see that its dual optimal 
$\thetab^*$ is a vector~\footnote{For simple notation, we denote $\thetab^*$ by 
a dual optimal solution given $\lambda$. 
We use notation $\thetab^*(\lambda)$ when we refer to a
specific $\lambda$.}
that is closest to $\frac{\by}{\lambda}$ among the ones that satisfy
all the constraints, and such a solution
can be obtained by projecting $ \frac{\by}{\lambda}$
onto the set of constraints $\bF' \equiv \left\{ \left\| \bX_{\bg}^T \thetab \right\|_2 \leq  \sqrt{n_\bg}, 
\, \forall  \bg \in \mathcal{G}' \right\} $.
In other words, $\thetab^* = P_{\bF'(\by / \lambda)}$, 
where $P_{\bF'}$ denotes the projection operator onto $\bF'$.

The  KKT conditions of \eqref{eq:dual_gr_lasso} \cite[]{wang2013lasso} are given by
\begin{equation}
\begin{aligned}
\by &= \bX \bbeta^* + \lambda \thetab^*,\\
\bX_{\bg}^T \thetab^*    &= 
  \begin{cases}
  \sqrt{n_\bg} \frac{ \bbeta_{\bg}^*}{\left\|  \bbeta_{\bg}^* \right\|_2}, & \text{ if } 
  \bbeta_{\bg}^* \neq {\bf 0}, \\
 \sqrt{n_\bg} \bu, \left\| \bu \right\|_2 \leq 1,     & \text{ if } \bbeta_{\bg}^* = {\bf 0}.
  \end{cases}
\end{aligned}
\label{eq:grlasso_kkt}
\end{equation}

Based on \eqref{eq:grlasso_kkt}, 
one can 
obtain a screening rule as follows:
if $\left\| \bX_{\bg}^T \thetab^*   \right\|_2 < \sqrt{n_\bg}$, then $\bbeta_{\bg}^* = {\bf 0}$.
However,  it is still unusable because the
dual optimal $\thetab^*$ is unknown.
To address this problem, we estimate a
range of vectors, denoted by $\Thetab$,
that contains
$\thetab^*$
 based on $\thetab^*(\lambda_0)$, where $\lambda_0 \neq \lambda$.
Specifically, to estimate $\Thetab$,
 we use the fact that 
 $ P_{\bF'}$ 
is continuous and nonexpansive.
Finally,  the DPP screening rule for nonoverlapping group lasso is formulated as follows: if
$\sup_{\thetab \in \Thetab} \left\| \bX_{\bg}^T \thetab   \right\|_2 < \sqrt{n_\bg}$,
then $\bbeta_{\bg}^* = {\bf 0}$. 
By finding a closed-form solution for the left-hand side of 
the rule, i.e., $\sup_{\thetab \in \Thetab} \left\| \bX_{\bg}^T \thetab   \right\|_2$,
we can obtain screening rules for nonoverlapping group lasso.

\section{Overlapping Group Lasso Screening}

Now we develop overlapping group lasso screening rules.
The first challenge 
is that the groups are not separable, making independent tests infeasible. 
Second, it is also unclear how to make use of overlapping groups
in screening rules. Intuitively, coefficients are more likely to 
be zero if they are involved in more
groups; thus incorporating
overlapping groups into screening rules can help discard more
features.

We address the first challenge by considering only a set of groups
that is a subset of the tested group.
Consider the example
in Figure \ref{fig:ogroup_concept}.
Suppose we want to test if $\bbeta_{\bg_2} = {\bf 0}$, $\bg_2 = \{2,3,4\}$ given
three other groups: $\bg_1 = \{1,2,3\}$, 
$\bg_3 = \{3,4\}$, and $\bg_4 = \{4,5\}$.
In such a case,  we consider only 
$\bg_3$ for the screening test on $\bg_2$ because $\bg_3$
is included in $\bg_2$, allowing us to 
test $\bg_2$ independent of other groups;
we ignore $\bg_1$ and $\bg_4$  in testing $\bg_2$
because they involve coefficients not included in $\bg_2$, preventing
us from testing the groups independently.
For the second challenge, we derive new screening rules
that can exploit the overlapping groups
for minimizing the left-hand side of screening rules. 
In other words, they 
discard 
more features as the number of overlapping groups, inclusive 
of a tested group,
increases.

In $\S$\ref{subset:screen_cond}, we start with  
a condition for $\bbeta_{\bg}^* = {\bf 0}$ for overlapping group lasso
that contains a dual optimal. 
Based on this condition, we derive a screening rule by replacing 
a dual optimal with a range that includes the dual optimal in $\S$\ref{subset:screen_ogroup}.
Finally, in $\S$\ref{subsec:screen_alg}, we present  efficient algorithms 
for overlapping group lasso screening based on the screening rules obtained
in $\S$\ref{subset:screen_ogroup}. 

\subsection{Screening Condition for Overlapping Group Lasso}
\label{subset:screen_cond}
Let us start with the dual form of overlapping group lasso (see appendix for derivation of the dual form):
\begin{align}
\label{eq:dual_formul}
& \sup_{\thetab}  \frac{1}{2} \left\| \by \right\|_2^2  -\frac{\lambda^2}{2} \left\| \thetab - \frac{\by}{\lambda} \right\|_2^2 \\
& \mbox{subject to }    \bX^T \thetab  =   \bv, \nonumber 
\end{align}
where $\thetab$ is a vector of dual variables, and $\bv$
is a subgradient of 
$\sum_{\bg\in \mathcal{G}} \sqrt{n_\bg} \left\| \bbeta_{\bg} \right\|_2$ 
with respect to $\bbeta$. 
A dual optimal $\thetab^*$  can be obtained by  
projecting $\frac{\by}{\lambda}$
onto $\bF \equiv \left( \bX^T \thetab   =   \bv \right)$, 
denoted by $P_{\bF}\left(\frac{\by}{\lambda}\right)$.
We will use the ``non-expansiveness'' property of this projection operator to derive
screening rules in $\S$\ref{subset:screen_ogroup}.

Next we derive the KKT conditions for overlapping group lasso. 
Introducing $\bz = \by - \bX \bbeta$,  \eqref{equ:general} can be written as
\begin{align}
\label{eq:overlap_g_lasso2}
&\min_{\bbeta}  \frac{1}{2} \left\| \bz  \right\|_2^2 + 
\lambda \sum_{\bg\in \mathcal{G}}\sqrt{n_{\bg}} \left\| \bbeta_{\bg} \right\|_2 \\
&\mbox{subject to }  \bz = \by - \bX \bbeta. \nonumber
\end{align}
Then, a Lagrangian of \eqref{eq:overlap_g_lasso2} is 
\begin{equation}
L(\bbeta,\bz,\thetab) =  \frac{1}{2} \left\| \bz  \right\|_2^2 + 
\lambda \sum_{\bg\in \mathcal{G}} \sqrt{n_{\bg}} \left\| \bbeta_\bg \right\|_2 
+\lambda \mathbf{\thetab}^T \left( \by - \bX \bbeta - \bz \right),
\label{eq:overlap_g_lasso_l}
\end{equation}
and the  KKT conditions of \eqref{eq:overlap_g_lasso_l} are as follows:
\begin{align}
&0 \in \frac{ \partial L(\bbeta^{*},\bz^{*},\thetab^{*})}{\partial \bbeta_\bg} = -\lambda \bX_\bg^T \thetab^{*}  + 
\lambda \bv_\bg,
\label{eq:kkt1}\\
&0 = \nabla_{\bz} L(\bbeta^{*},\bz^{*},\thetab^{*}) = \bz^{*} - \lambda \thetab^{*} \label{eq:kkt2}, \\
&0 = \nabla_{\thetab} L(\bbeta^{*},\bz^{*},\thetab^{*}) =  \lambda \left( \by - \bX \bbeta^{*} - \bz^{*} \right)
 \label{eq:kkt3},
\end{align}
where $\bv_\bg$ is a subgradient of 
$\sum_{\bg\in \mathcal{G}} \sqrt{n_{\bg}} \left\| \bbeta_\bg \right\|_2$ 
with respect to $\bbeta_\bg$.
From \eqref{eq:kkt2} and \eqref{eq:kkt3}, we obtain a bridge 
between the primal and dual solutions:
\begin{equation}
\label{eq:primal_dual}
\lambda \thetab^{*} = \by - \bX \bbeta^{*}.
\end{equation}

Let us define two sets of groups that overlap with group $\bg$ as follows: 
\begin{align}
&\bar{\mathcal{G}}_1 = \{\bh: \bh\in \mathcal{G} - \bg, \bh  \subseteq \bg, \bh \cap \bg \neq \emptyset \}, 
\label{eq:group_def1}\\
&\bar{\mathcal{G}}_2 = \{\bh: \bh\in \mathcal{G} - \bg,  \bh \nsubseteq \bg, \bh \cap \bg \neq \emptyset \}, 
\end{align}
where
$\bar{\mathcal{G}}_1$ and 
$\bar{\mathcal{G}}_2$ are sets of groups 
overlapping with $\bg$, and
$\bar{\mathcal{G}}_1$ includes the groups that are inclusive of $\bg$.
Then, 
we denote 
$\bv_\bg =\sqrt{n_\bg} [\gamma_1, \ldots, \gamma_{n_{\bg}}]^T$, where $\gamma_j$
is given by 
\begin{equation}
\label{eq:subgrad}
\gamma_j = 
u_j
 + \sum_{j\in \bh, \bh \in \bar{\mathcal{G}}_1} w_{j}
 + \sum_{j\in \bh, \bh \in \bar{\mathcal{G}}_2} s_j,
\end{equation}
where $u_j$ is a subgradient of $\left\| \bbeta_{\bg} \right\|_2$ 
with respect to $\beta_j$;
$w_j$ and $s_j$ 
are subgradients of $\left\| \bbeta_{\bh} \right\|_2$
with respect to $\beta_j$, where
$\bh$ belongs to  $\bar{\mathcal{G}}_1$ and $\bar{\mathcal{G}}_2$, respectively.
The definition of $\ell_2$ norm subgradient, i.e.,
$\left\| \bu_\bg \right\|_2 \leq 1$, 
$\left\| \bw_\bh \right\|_2 \leq 1$, and
$\left\| \bs_\bh \right\|_2 \leq 1$, gives us 
\begin{equation}
 \label{eq:subgrad_cond}
\sqrt{\sum_{j\in \bg} u_j^2}  = 
\sqrt{\sum_{j \in \bg} 
\left(\gamma_j - \sum_{j\in \bh, \bh \in \bar{\mathcal{G}}_1} w_j -
 \sum_{j\in \bh, \bh \in \bar{\mathcal{G}}_2} s_j \right)^2} \leq 1,
\end{equation}
where the equality holds when $\bbeta_\bg^* \neq {\bf 0}$.
Plugging  \eqref{eq:kkt1} into \eqref{eq:subgrad_cond},
 a sufficient condition  for $\bbeta_\bg^* = {\bf 0}$
 is given by 
{\scriptsize
\begin{equation}
\label{eq:zero_cond2}
\min_{\bw, \bs:
\left\| \bw_\bh \right\|_2 \leq 1, 
\left\| \bs_\bh \right\|_2 \leq 1 
}   \sqrt{\sum_{j \in \bg} \left(    \bx_j^T \thetab^{*}   - 
\sum_{j\in \bh, \bh \in \bar{\mathcal{G}}_1} \sqrt{n_{\bh}} w_j -
\sum_{j\in \bh,\bh \in \bar{\mathcal{G}}_2} \sqrt{n_{\bh}} s_j \right)^2} < \sqrt{n_{\bg}}.
\end{equation}
}

To screen each group independently 
(i.e., test using only the coefficients in group $\bg$), we set 
$\bs_\bh = 0$, 
for all $ \bh \in \bar{\mathcal{G}}_2$. 
This is a valid subgradient because it always satisfies $\left\| \bs_{\bh} \right\|_2 \leq 1$.
Given  subgradients of the groups inclusive of $\bg$,
we have the following screening condition for $\bbeta_\bg^*= \mathbf{0}$: 
if $b_\bg < \sqrt{n_{\bg}}$,  then $\bbeta^*_\bg = \mathbf{0}$, 
where $b_\bg$ is defined by
\begin{align}
\label{eq:mini_opt_problem}
b_\bg \equiv  \min_{\bw: \left\| \bw_\bh \right\|_2 \leq 1} & \sqrt{\sum_{j \in \bg}
\left(  \bx_j^T \thetab^{*} 
   - \sum_{j\in \bh, \bh \in \bar{\mathcal{G}}_1} \sqrt{n_{\bh}} w_j  \right)^2}. 
\end{align}
Note that 
$b_{\bg}$ is an upper bound on the left-hand side of \eqref{eq:zero_cond2}
due to the fixed $\bs_\bh = {\bf 0}$, for all $\bh \in \bar{\mathcal{G}}_2$.
If $b_\bg < \sqrt{n_{\bg}}$, then \eqref{eq:zero_cond2} holds, 
and thus $\bbeta^*_\bg = \mathbf{0}$.

\subsection{Screening Rules for Overlapping Group Lasso}
\label{subset:screen_ogroup}
So far, we have derived a condition for $\bbeta^*_\bg = \mathbf{0}$;
however, it is not yet usable for screening because $\thetab^{*}$ is unknown. 
Thus, by
following the DPP approach by Wang et al. \cite[]{wang2013lasso},
we first find 
a region $\Thetab$ that contains
$\thetab^{*}$.

In \eqref{eq:dual_formul}, an optimal $\thetab^*$ is 
the projection of $\frac{\by}{\lambda}$ onto the constraint $\bF$:
\[
\thetab^{*} = P_{\bF}\left(\frac{\by}{\lambda}\right) = \argmin_{\thetab \in \bF} \left\| \thetab - \frac{\by}{\lambda} \right\|_2,
\]
where $P_{\bF}$ is the  projection operator with $\bF$
which is a nonempty, closed convex subset of a Hilbert space 
($\bF$ is nonempty 
because $\mathbf{0} \in \bF$, and 
closed convex because
it is an intersection of closed half-spaces). 
Thus, we can use the ``non-expansiveness'' property of $P_{\bF}$ \cite[]{bertsekas2003convex}, given by
\begin{equation}
\label{eq:nonex}
\left\| P_{\bF}\left( \frac{\by}{\lambda} \right) - P_{\bF}\left(\frac{\by}{\lambda_0}\right)  \right\|_2 
=
\left\| \thetab^{*}(\lambda) - \thetab^{*}(\lambda_0)  \right\|_2 
 \leq \left\| \frac{\by}{\lambda} - \frac{\by}{\lambda_0} \right\|_2,
\end{equation}
where $\lambda_0$ 
is a tuning parameter ($\lambda_0 > \lambda$), and
$\thetab^*(\lambda) \equiv \thetab^*$,
and $\thetab^*(\lambda_0)$ is a dual optimal solution given $\lambda_0$.
Here \eqref{eq:nonex} shows that 
$\thetab^*(\lambda)$ lies within a sphere $\Thetab$ centered at 
$\thetab^*(\lambda_0)$ with a radius of 
$\rho = \left\| \frac{\by}{\lambda} - \frac{\by}{\lambda_0} \right\|_2$.
Based on this, we can represent $\thetab^*(\lambda) =  \thetab^*(\lambda_0) + \br$,
where $\left\| \br \right\|_2 \leq \rho$.
By plugging it  into 
\eqref{eq:mini_opt_problem} and maximizing the objective over $\br$, we have
the following screening rule: if $b'_\bg < \sqrt{n_{\bg}}$, then $\bbeta^* = {\bf 0}$, where $b'_\bg$
is defined by
{\footnotesize 
\begin{align}
\label{eq:mini_opt_problem2}
b'_\bg \equiv  \sup_{\br: \left\|\br \right\|_2 \leq \rho}\min_{\bw: \left\| \bw_\bh \right\|_2 \leq 1 } & \sqrt{\sum_{j \in \bg}
\left(   \bx_j^T \left\{\thetab^{*}(\lambda_0)+\br \right\}
   - \sum_{j\in \bh, \bh \in \bar{\mathcal{G}}_1} \sqrt{n_{\bh}} w_j  \right)^2}. 
\end{align}
}
Notice that $b'_\bg$ is an upper bound on $b_\bg$, and thus
$b'_\bg < \sqrt{n_{\bg}} \Rightarrow b_\bg < \sqrt{n_{\bg}}
\Rightarrow \bbeta^* = {\bf 0}$. 
With a little bit of algebra, 
we get our screening rule for overlapping group lasso.

\begin{theorem}
For the overlapping lasso problem, suppose that we are given an optimal 
dual solution
$\thetab^{*}(\lambda_0)$.
Then for $\lambda < \lambda_0$, $\bbeta_\bg^*(\lambda) = \mathbf{0}$ if
{\scriptsize  
\begin{equation}
\label{eq:theorem_screen}
 \min_{\bw_\bh: \left\| \bw_\bh \right\|_2 \leq 1}   
\sqrt{ \sum_{j \in \bg}
 \left(
     \bx_j^T \thetab^*(\lambda_0)  
   - \sum_{j\in \bh, \bh \in \bar{\mathcal{G}}_1} \sqrt{n_\bh} w_{j}  
   \right)^2}
   < \sqrt{n_\bg} -  \left\| \bX_\bg \right\|_F \left\| \by \right\|_2  \left| \frac{1}{\lambda} - \frac{1}{\lambda_0} \right|.
\end{equation}
}
\label{theorem:screen}
\end{theorem}
\begin{proof}
See Appendix \ref{appendix:proof_theorem1}.
\end{proof}

Using the bridge between the primal and dual in \eqref{eq:primal_dual}, 
we can also obtain a screening rule in a primal form. 
\begin{theorem}
For the overlapping lasso problem, suppose that we are given an optimal 
solution
$\bbeta^{*}(\lambda_0)$. 
Then for $\lambda < \lambda_0$, $\bbeta_\bg^*(\lambda) = \mathbf{0}$ if
{\tiny
\begin{equation}
\label{eq:theorem_screen2}
 \min_{\bw_\bh: \left\| \bw_\bh \right\|_2 \leq 1}   
\sqrt{ \sum_{j \in \bg}
 \left(
    \bx_j^T \frac{\by - \bX \bbeta^*(\lambda_0)}{\lambda_0}   
   - \sum_{j\in \bh, \bh \in \bar{\mathcal{G}}_1} \sqrt{n_\bh} w_{j}  
   \right)^2}
   < \sqrt{n_\bg} -  \left\| \bX_\bg \right\|_F \left\| \by \right\|_2  \left| \frac{1}{\lambda} - \frac{1}{\lambda_0} \right|.
\end{equation}
}
\label{theorem:screen2}
\end{theorem}

We also note that Theorem \ref{theorem:screen2}
can be employed to solve lasso problems following a $\lambda$ path
$\{\lambda_1,\lambda_2, \ldots, \lambda_T\}$
in a descending order. 
The $\lambda$ path is determined {\it a priori}, and one may choose
linearly, geometrically, or logarithmically spaced $\lambda$ values.
In the sequential version of screening, we first perform screening 
with $\lambda_1$ using $\lambda'$ with $\bbeta^*(\lambda') = {\bf 0}$.
We then run a solver using the remaining coefficients with their corresponding groups;
its results become $\bbeta^*(\lambda_{1})$. 
Now, using $\lambda_1$ with $\bbeta^*(\lambda_{1})$, we
perform screening for $\lambda_2$. 
We repeat the above procedure
for all remaining $\lambda$ parameters.
The following theorem shows sequential screening rule, where
a  screening rule for $\lambda_t$
is constructed based on
$\bbeta^{*}(\lambda_{t-1})$, $t\geq 2$.

\begin{theorem}
For the overlapping lasso problem with a $\lambda$ path
$\{\lambda_1, \ldots, \lambda_T\}$, $\lambda_{t-1}>\lambda_{t}, t=2,\ldots,T$, suppose that we are given an optimal 
solution
$\bbeta^{*}(\lambda_{t-1})$. 
Then, $\bbeta_\bg^*(\lambda_t) = \mathbf{0}$ if
{\tiny  
\begin{equation}
\label{eq:theorem_screen3}
 \min_{\bw_\bh: \left\| \bw_\bh \right\|_2 \leq 1}   
\sqrt{ \sum_{j \in \bg}
 \left(
   \bx_j^T \frac{\by - \bX \bbeta^*(\lambda_{t-1})}{\lambda_{t-1}}    
   - \sum_{j\in \bh, \bh \in \bar{\mathcal{G}}_1} \sqrt{n_\bh} w_{j}  
   \right)^2}
   < \sqrt{n_\bg} -  \left\| \bX_\bg \right\|_F \left\| \by \right\|_2  \left| \frac{1}{\lambda_t} - \frac{1}{\lambda_{t-1}} \right|.
\end{equation}
}
\label{theorem:screen3}
\end{theorem}
We omit the proofs for Theorem \ref{theorem:screen2} and Theorem \ref{theorem:screen3}
because it is straightforward to derive them from 
Theorem \ref{theorem:screen}.

\subsection{Screening Algorithms for Overlapping Group Lasso}
\label{subsec:screen_alg}
To use Theorems \ref{theorem:screen}, \ref{theorem:screen2}, \ref{theorem:screen3}
we need an efficient way to obtain the left-hand side.
Instead of solving the left-hand side directly, we  minimize an upper bound on the left-hand side
because it can be quickly solved. 
Any upper bounds give us a valid screening rule,
but the tighter the bound, the better the screening efficiency 
(it discards more features).
Note that the goal of screening is to speed up  optimization, and thus
we intend to present a simple yet efficient algorithm.

We first present our algorithm and then verify that 
it minimizes an upper bound on the left-hand side.
We adopt a simple coordinate descent-type approach, where each group is
used for minimization one at a time. 
Suppose that we perform screening on group $\bg$. 
We start with 
making two variables: 
$l \leftarrow 0$, $\ba \leftarrow \bg$,
and a set of overlapping groups 
$\{\bh_1,\ldots,\bh_K: \bh_k \in \bar{\mathcal{G}}_1, k=1, \ldots, K\}$ 
(any ordering works for our purpose).
For each group $\bh_k$, 
we take the intersection between 
$\bh_k$ and $\ba$, i.e., $\bh'_k \leftarrow \bh_k \cap \ba$. 
If $\bh'_{k} = \emptyset$, we skip $\bh_{k}$ and proceed to the next $\bh_{k+1}$.
If $\bh'_{k} \neq \emptyset$, we take the following procedure.
If $ \left\|\bX_{\bh'_k} \thetab^*(\lambda_0) \right\|_{2}  
\leq
 \sqrt{n_{\bh_k}}$, 
we set $l \leftarrow l $; 
otherwise $l \leftarrow l + z$, 
where 
$z = \sum_{j\in \bh'_{k}} \left\{\bx_{j}^{T} \thetab^*(\lambda_0) -  \sqrt{n_{\bh_k}} w_{j}\right\}^{2}$,
and 
$\{w_{j}: j\in \bh'_{k}\}$
is determined by the following algorithm.
\begin{enumerate}
\item Set $d = 1$. 
\item For each $j \in \bh'_{k}$, we compute
\begin{equation}
w_{j} = 
  \begin{cases}
   \frac{\bx_{j}^{T} \thetab^*(\lambda_0)}{\sqrt{n_{\bh_k}}}  & \text{, if } 
   \left| \frac{\bx_{j}^{T} \thetab^*(\lambda_0)}{\sqrt{n_{\bh_k}}} \right| \leq \sqrt{d} \\
\mbox{sign}\left\{\bx_{j}^{T} \thetab^*(\lambda_0) \right\} \sqrt{d}      & \text{, otherwise},
  \end{cases}
  \label{eq:min_subalg}
\end{equation}
and update $d \leftarrow d - w_{j}^{2}$.
\end{enumerate}

We then 
set $\ba \leftarrow \ba - \bh'_k$, 
and iterate this procedure over all groups in $\bar{\mathcal{G}}_1$.
Finally,
a minimized left-hand side for the screening rules is 
$\sqrt{l+ \left\| \bx_{\ba}^{T} \thetab^*(\lambda_0) \right\|_{2}^2}$.
This algorithm
 in a sequential  setting for overlapping group lasso
(screening for $\lambda_t$ given $\lambda_{t-1}$, and 
$\lambda_t \in \{\lambda_1, \ldots, \lambda_T\}$)
is summarized in Algorithm \ref{alg:screening}.

\begin{algorithm}[t]
{\scriptsize 
 \KwIn{$\bX, \by,  \lambda_{t-1}, \lambda_t \, (\lambda_{t-1} > \lambda_{t}),
 \mathcal{G}$, 
 $\thetab^*(\lambda_{t-1}) \left(= \frac{\by-\bX \bbeta^*(\lambda_{t-1})}{\lambda_{t-1}}\right)$ }
 \KwOut{$\mathcal{T}$ (a set of groups with potential nonzero coefficients)}
 $\mathcal{T} \leftarrow \emptyset$\;
 \For{$\bg \in \mathcal{G}$}{
$\bar{\mathcal{G}}_1 = \{\bh: \bh\in \mathcal{G} - \bg, \bh  \subseteq \bg, \bh \cap \bg \neq \emptyset \}$\;
 $\ba \leftarrow \bg$\; 
 $l \leftarrow 0$\; 
   \For{$\bh \in  \bar{\mathcal{G}}_1$}{ 
    $\bh' \leftarrow \bh \cap \ba$\;
  \If{$\left\| \bX_{\bh'}^T \thetab^*(\lambda_{t-1}) \right\|_2 >  \sqrt{n_{\bh}}$}{
  $d \leftarrow 1$\;
 \For{$j \in \bh'$}{
 	\If{$\left| \bx_{j}^{T} \thetab^*(\lambda_{t-1}) \right| \leq \sqrt{d n_{\bh}} $}{
	   		$d \leftarrow d - \left(\frac{\bx_{j}^{T} \thetab^*(\lambda_{t-1})}{\sqrt{n_{\bh}}} \right)^{2}$\;
 	}\Else{ 
 		$l \leftarrow l + \left[\bx_{j}^{T} \thetab^*(\lambda_{t-1}) -  \sqrt{d n_{\bh}} \mbox{sign}\left\{\bx_{j}^{T} \thetab^*(\lambda_{t-1}) \right\} \right]^{2}$\;
 		break\;
  	}
 }
  }
    $\ba \leftarrow \ba  -  \bh'$\; 
  }
  \If{$ \sqrt{l +  \left\|\bX_{\ba}^T \thetab^*(\lambda_{t-1})\right\|_2^2} \geq \sqrt{n_\bg} -  \left\| \bX_\bg \right\|_F \left\| \by \right\|_2 
  \left| \frac{1}{\lambda_t} - \frac{1}{\lambda_{t-1}} \right|$}{
  $\mathcal{T} \leftarrow \mathcal{T}  \cup  gi$\;
  }
 }
 }
 \caption{Screening for overlapping group Lasso}
 \label{alg:screening}
\end{algorithm}

Now, we show that this algorithm minimizes 
an upper bound on 
the left-hand side of Theorem \ref{theorem:screen}.
The key idea is that, at the $(k+1)$-th iteration, 
we minimize an upper bound on 
the bound obtained
 at the $k$-th iteration.
We denote $\ba$ by the set of coefficients to be processed,
and  $k$ by the iteration counter,
initialized by $\ba \leftarrow \bg$ and $k=1$.
At the $k$-th iteration,  the squared left-hand side is bounded 
as follows:
{\footnotesize 
\begin{align}
&\min_{\bw_{\bh}, \forall \bh \in \bar{\mathcal{G}}_1}   
\left[
\sum_{j \in \bg}
 \left(
\bx_{j}^T \thetab^*(\lambda_0)    
   - \sum_{j\in \bh, \bh \in \bar{\mathcal{G}}_1} \sqrt{n_\bh} w_j  
   \right)^2
   \right]
   \label{eq:screening_bound1} \\
&  \leq   
\min_{\bw_{\bh}, \forall \bh \in \bar{\mathcal{G}}_1 - \bh_k}   
\left[
\sum_{j \in \bg - \bh_k}
 \left(
   \bx_j^T \thetab^*(\lambda_0)   
   - \sum_{j\in \bh, \bh \in \bar{\mathcal{G}}_1 - \bh_k} \sqrt{n_\bh} w_j  
   \right)^2
      \right] 
      \label{eq:screening_bound2}\\
& \ \ \ \ \ \ \ \ \ \ \ \ \ \ \ \   \ \ \ \ \ \ \ \ \ \ \ \ \ \ \ \ +
\min_{\bw_{\bh_k}} 
 \left\|
\bX_{\bh_k}^T \thetab^*(\lambda_0)   
   - \sqrt{n_{\bh_k}} \bw_{\bh_k}
   \right\|_2^2.
   \nonumber
\end{align}
}
To obtain the upper bound in \eqref{eq:screening_bound2}, we set $w_j = 0$ for all 
$\{j: j \in   \bh \cap \bh_k, \bh \in \bar{\mathcal{G}}_1 -\bh_k\}$.
Let us fix $\{\bw_{\bh}: \bh \in \bar{\mathcal{G}}_1 - \bh_k\}$
and then find  
the bound in \eqref{eq:screening_bound2}.
Since the first term  is a constant due to fixed 
$\{\bw_{\bh}\}$, we find $z$ that bounds the second term:
$$\min_{\bw_{\bh_k}} 
 \left\|
\bX_{\bh_k}^T \thetab^*(\lambda_0)   
   - \sqrt{n_{\bh_k}} \bw_{\bh_k}
   \right\|_2^2 \leq z.$$
Based on the subgradient condition $\left\| \bw_{\bh_k} \right\|_2 \leq 1$,
if $\left\| \bX_{\bh_k}^T \thetab^*(\lambda_0) \right\|_2 \leq \sqrt{n_{\bh_k}}$,
we set $ \left\|
\bX_{\bh_k}^T \thetab^*(\lambda_0)   
   - \sqrt{n_{\bh_k}} \bw_{\bh_k}
   \right\|_2^2 = 0 = z$;
otherwise, we find an upper bound $z$ 
using the  simple coordinate descent-type procedure
with \eqref{eq:min_subalg}. 
It is easy to see that the 
procedure with \eqref{eq:min_subalg}
satisfies the subgradient condition $\left\| \bw_{\bh_k} \right\|_2 \leq 1$,
and thus $z$ is a valid upper bound.
Then, we set $\ba \leftarrow \ba - \bh_k$ because
$\bh_k$ is processed, and set $l \leftarrow z$.

Subsequently, we get an 
upper bound on \eqref{eq:screening_bound2}:
{\scriptsize
\begin{align}
&\min_{\bw^{\bh}, \forall \bh \in \bar{\mathcal{G}}_1 - \bh_k}   
\left[
\sum_{j \in \bg - \bh_k}
 \left(
   \bx_j^T \thetab^*(\lambda_0)   
   - \sum_{j\in \bh, \bh \in \bar{\mathcal{G}}_1 - \bh_k} \sqrt{n_\bh} w_j  
   \right)^2
   \right]
 + l   
   \nonumber \\
& \leq   
\min_{\bw^{\bh}, \forall \bh \in \bar{\mathcal{G}}_1 - \bh_k - \bh_{k+1}}   
\left[
\sum_{j \in \bg - \bh_k-\bh_{k+1}}
 \left(
   \bx_j^T \thetab^*(\lambda_0)   
   - \sum_{j\in \bh, \bh \in \bar{\mathcal{G}}_1 - \bh_k -\bh_{k+1}} \sqrt{n_\bh} w_j  
   \right)^2
   \right]
  \\
& \ \ \ \ \ \ \ \  \ \ \ \  \ \ \ \ \ \ \ \  \ \ \ \ +\min_{\bw_{\bh'_{k+1}}} 
 \left\|
\bX_{\bh'_{k+1}}^T \thetab^*(\lambda_0)   
   - \sqrt{n_{\bh_{k+1}}} \bw_{\bh'_{k+1}}
   \right\|_2^2 + l, 
\nonumber
\label{eq:screening_bound3} 
\end{align}}
where $\bh'_{k+1} = \bh_{k+1} \cap \ba$. 
Fixing $\{\bw_{\bh}: \bh \in \bar{\mathcal{G}}_1 - \bh_k - \bh_{k+1}\}$, 
and setting 
$w_j = 0, \, \forall j \in   \bh \cap \bh_k, \bh \in \bar{\mathcal{G}}_1 -\bh_k - \bh_{k+1}$,
we minimize an upper bound $z$ on $ \left\|
\bX_{\bh'_{k+1}}^T \thetab^*(\lambda_0)   
   - \sqrt{n_{\bh_{k+1}}} \bw_{\bh'_{k+1}}
   \right\|_2^2$ using the procedure with \eqref{eq:min_subalg},
   and $l \leftarrow l + z$.
We iterate this procedure 
over all groups in $\bar{\mathcal{G}}_1$, i.e., $\{\bh \in \bar{\mathcal{G}}_1\}$,
resulting in an upper bound on the left-hand side as follows:

{\footnotesize 
\begin{align}
\min_{\bw^{\bh}, \forall \bh \in \bar{\mathcal{G}}_1 - \bh_1 - \ldots - \bh_{K}}   
\left[
\sum_{j \in \bg -  \bh_1 - \ldots - \bh_{K} }
 \left(
   \bx_j^T \thetab^*(\lambda_0)   
   - \sum_{j\in \bh, \bh \in \bar{\mathcal{G}}_1 - \bh_1 - \ldots - \bh_{K}} \sqrt{n_\bh} w_j  
   \right)^2
   \right]
 +  l. 
\nonumber
\end{align}}
By setting
$w_j = 0$ for all  $j \in   \bh \cap \bh_k, \bh \in \bar{\mathcal{G}}_1 - \bh_1 - \ldots - \bh_{K}$,
we get an 
upper bound on the 
squared 
left-hand side, i.e., $l+ \left\| \bx_{\ba}^{T} \thetab^*(\lambda_0) \right\|_{2}^2$;
taking 
the square root on it,  we obtain the left-hand side of Theorem \ref{theorem:screen}.

We note that the DPP screening rule for nonoverlapping group lasso (GDPP)
\cite[]{wang2013lasso}
is a special case of the proposed screening rules for overlapping group lasso.
In Theorem \ref{theorem:screen}, 
by setting $w_j = 0$ for all $j \in \bh$, we obtain GDPP, where 
its left-hand side is an upper bound on that of our screening rules.
This  implies  that
GDPP is also an exact screening rule for overlapping group lasso; 
however, GDPP would not be as efficient as Theorem \ref{theorem:screen} due to the lack of 
degrees of freedom to decrease its left-hand side. 
It is surprising that GDPP is applicable to overlapping group lasso
because GDPP is derived under the assumption that groups do not overlap.

In practice, finding $\bar{\mathcal{G}}_1$ can be an algorithmic bottleneck
(line 3 in Algorithm \ref{alg:screening}). 
To search for $\bar{\mathcal{G}}_1$ efficiently, we used a simple algorithm.
We first sort the groups based on the smallest index of each group, resulting
in $\mathcal{G} = \{\bg^{(1)}, \ldots, \bg^{(M)}\}$.
To perform a screening test on $\bg^{(m)}$, 
we find its $\bar{\mathcal{G}}_1$ 
by searching for the groups between $\bg^{(m+1)}$ and
$\bg^{(m+W)}$, where
$W$ is the user-defined window size.
With larger $W$, we may discard more features, 
but the computational complexity for screening
increases linearly in $W$.

\subsubsection{Screening Algorithm for Sparse Overlapping Group Lasso}

\begin{algorithm}[t]
{\scriptsize 
 \KwIn{$\bX, \by,  \lambda_{t-1}, \lambda_t \, (\lambda_{t-1} > \lambda_{t}),
 \mathcal{G}$, 
 $\thetab^*(\lambda_{t-1}) \left(= \frac{\by-\bX \bbeta^*(\lambda_{t-1})}{\lambda_{t-1}}\right)$ }
 \KwOut{$\mathcal{T}$ (a set of groups with potential nonzero coefficients)}
 $\mathcal{T} \leftarrow \emptyset$\;
 \For{$\bg \in \mathcal{G}$}{
\If{$\left| \mathcal{G} \right| \geq 2$}{
 $l \leftarrow 0$\; 
   \For{$j \in  \mathcal{G}$}{ 
 	\If{$\left| \bx_{j}^{T} \thetab^*(\lambda_{t-1}) \right| > 1 $}{
 		$l \leftarrow l + \left\{\bx_{j}^{T} \thetab^*(\lambda_{t-1}) -   \mbox{sign}\left(\bx_{j}^{T} \thetab^*(\lambda_{t-1}) \right) \right\}^{2}$\;
  	}
  }
  } \Else{
 		$l \leftarrow  \left\{\bx_{j}^{T} \thetab^*(\lambda_{t-1})  \right\}^{2}$\;
  }
  \If{$ \sqrt{l} \geq \sqrt{n_\bg} -  \left\| \bX_\bg \right\|_F \left\| \by \right\|_2 
  \left| \frac{1}{\lambda_t} - \frac{1}{\lambda_{t-1}} \right|$}{
  $\mathcal{T} \leftarrow \mathcal{T}  \cup  gi$\;
  }
 }
 }
 \caption{Screening for sparse overlapping group lasso}
 \label{alg:sogrp_screening}
\end{algorithm}

Sparse overlapping group lasso is
a special case of overlapping group lasso that
includes $\ell_1$ penalty, defined by
\begin{equation}
\label{equ:sparse_og_lasso}
\min_{\bbeta} \frac{1}{2}
\lVert \by - \bX\bbeta  \rVert_2^2 
+ \lambda_1 \left\| \bbeta \right\|_1 +
\lambda_2  \sum_{\bg\in \mathcal{G}} \sqrt{n_\bg} \left\| \bbeta_{\bg} \right\|_2. 
\end{equation}
For this model, we  
provide a simple and fast algorithm
by considering only the individual coefficients
in $\ell_1$ penalty for $\bar{\mathcal{G}}_1$.
Note that individual features can be considered as
groups of size one, inclusive of  other groups. 
Substituting $\bar{\mathcal{G}}_1$ in line 3 in 
Algorithm \ref{alg:screening} 
by $\bar{\mathcal{G}}_1 = \{j : j \in \bg\}$, we obtain 
a screening algorithm for sparse overlapping group lasso,
summarized in Algorithm \ref{alg:sogrp_screening}.
It is worthwhile to mention that
Algorithm \ref{alg:sogrp_screening}
is significantly faster than Algorithm \ref{alg:screening} 
due to the lack of set operations in Algorithm \ref{alg:screening}.
Furthermore, in our experiments, we observed that 
Algorithm \ref{alg:sogrp_screening} discards similar numbers of 
features to Algorithm \ref{alg:screening} on various datasets.
Thus, if the model in \eqref{equ:sparse_og_lasso} is utilized in 
an application,
Algorithm \ref{alg:sogrp_screening} would be appealing
in terms of both its screening rejection power and speed.

\begin{algorithm}[t]
{\scriptsize 
 \KwIn{$\bX, \by, \mathcal{G}, r$ ($0<r<1$ is a common ratio for a geometric series)
 }
 \KwOut{$\lambda'$ that discards all features}
 $\lambda_{1} = \max_{\bg \in \mathcal{G}}{\frac{1}{\sqrt{n_\bg}}\left\|\bX_{\bg}^T\by\right\|_2}$\;
 \For{$t=2$ to $T$}{
$\lambda_{t} = \lambda_{t-1} r$\;
$\thetab^{*}(\lambda_{t-1}) = \frac{\by}{\lambda_{t-1}}$\;
  $\mathcal{T} \leftarrow$ 
  Algorithm\ref{alg:screening}$(\bX,\by,\lambda_{t},\lambda_{t-1},\mathcal{G},\thetab^{*}(\lambda_{t-1}))$\;
  \If{$\mathcal{T} \neq \emptyset$}{
  $\lambda' = \lambda_{t-1}$\;
  break\;
  }
 }
 }
 \caption{Finding a small $\lambda$ that sets all coefficients to zero}
 \label{alg:lambdamax}
\end{algorithm}

\paragraph{Remark}
For nonoverlapping group lasso,
we find $\lambda_{max}$ 
(the smallest $\lambda$ that sets $\bbeta^* = {\bf 0}$) 
as follows:
$\lambda_{max} = \max_{\bg \in \mathcal{G}}{\frac{1}{\sqrt{n_\bg}}\left\|\bX_{\bg}^T\by\right\|_2}$.
However,  this is  
not necessarily the smallest $\lambda$ for zero solutions for overlapping group lasso.
In fact, it is nontrivial to compute $\lambda_{max}$ for overlapping group lasso
because different groups are coupled through overlaps, preventing 
us from using the simple equation above. 
Instead, using a screening algorithm, we can find a
small $\lambda$ that sets all coefficients to zero, denoted by $\lambda'$.
The key idea is that we decrease $\lambda$
following a sequence
until all coefficients are discarded by a screening algorithm.
We denote $\lambda'$
by the smallest $\lambda$ that sets $\bbeta^* = \bf{0}$ in our regularization path.
This technique is summarized in Algorithm \ref{alg:lambdamax}
with a geometric sequence of $\lambda$ parameters.

\section{Experiments}

We demonstrate the efficiency of the proposed screening algorithms
in terms of the screening rejection ratio 
and the speed.
The rejection ratio is defined by the
ratio of discarded 
coefficients to true zero coefficients, obtained by 
an overlapping group lasso solver without screening.
Here we refer Algorithm \ref{alg:screening} to 
OLS 
and Algorithm \ref{alg:sogrp_screening} to
SOLS. 

To the best of our knowledge, there are no existing 
screening rules for overlapping group lasso;
however, we showed that 
GDPP can also be used
for overlapping group lasso as an exact screening rule.
Therefore, as a baseline
screening algorithm, we used GDPP. Comparing 
OLS\footnote{We 
set the window size to 50 
to search for the overlapping groups in $\bar{\mathcal{G}}_1$.} and SOLS against GDPP, we  investigate the benefits of
using overlapping groups for screening because GDPP 
makes no use of  overlapping groups.

We used the following three image datasets 
\footnote{from http://www.csie.ntu.edu.tw/$\sim$cjlin/libsvmtools/datasets/}
and one genome dataset for our experiments,
chosen to cover the cases where
$N >> J$, $N<<J$, and $N \approx J$, and a real-world application 
in genetics. 
{\bf a)} The PIE image 
database \cite[]{sim2002cmu}, which contains
11,554 face images of 68 people under different poses, illumination conditions, and expressions. 
The images are represented by $\bD \in \mathbb{R}^{11554 \times 1024}$.
We generated the response vector $\by \in \mathbb{R}^{11554 \times 1}$  by randomly choosing a feature
in $\bD$, and the rest of the features are concatenated to be the design matrix 
$\bX  \in \mathbb{R}^{11554 \times 1023}$.
{\bf b)} The Alzheimer's disease (AD) dataset \cite[]{zhang2013integrated}, 
which contains
541 AD individuals, represented by 
511,997  single nucleotide polymorphisms (the most common 
genetic variants).
For the same individuals, the AD dataset contains 
40,638 expression levels for known and predicted genes,
splice variants,
miRNAs, and non-coding RNA sequences in 
the brain region of the cerebellum.
We used the genetic information for
$\bX  \in \mathbb{R}^{541 \times 511997}$, and
randomly choose a gene expression for
$\by  \in \mathbb{R}^{541 \times 1}$.
This is a typical experimental setting for 
expression quantitative mapping \cite[]{lee2012leveraging}, where the goal is 
to identify genetic variants that affect 
gene expression levels.
{\bf c)} The COIL database \cite[]{nayar1996columbia}  contains 
color images of 100 objects. For each object, 
72 images are taken with different angles.
We selected object 10, whose image data are represented by 
$\bD \in \mathbb{R}^{72 \times 49152}$. 
$\bX$ and $\by$ are generated in the same way as
the PIE dataset.
{\bf d)} The digit recognition data (GISETTE) \cite[]{guyon2004result}, 
which contains 6,000 digit images of
four and nine. Each sample contains
5,000 features, where $70\%$ of features are constructed from
MNIST dataset \cite[]{lecun1998gradient}, and $30\%$ 
of them are artificially generated following a 
distribution of true features, resulting in 
$\bX  \in \mathbb{R}^{6000 \times 5000}$. 
The response vector
$\by \in \mathbb{R}^{6000 \times 1}$ consists of binary labels, 
indicating the class (either four or nine)
of each sample. 

Furthermore, we test screening algorithms on the sparse
overlapping group lasso in 
\eqref{equ:sparse_og_lasso} with $\lambda_1 = \lambda_2$
under different group structures.
We chose this model 
because 
it  induces individual level sparsity, allowing us to 
capture complex patterns of nonzero coefficients.
Moreover, we generated group structures based on feature locality
because features located nearly are often associated with 
an output vector jointly
in image or genome datasets.
The group structures used in our experiments are as follows:

\begin{enumerate}
\item  {\bf $\ell_1$ + nonoverlap groups}:
$\mathcal{G}$ contains nonoverlapping groups, where 
each group consists of 20 consecutive features.

\item {\bf $\ell_1$  + tree structure groups}:
$\mathcal{G}$ contains tree structured groups with four levels, where 
from the root to the leaves, the groups  consist of 20/15/10/5 consecutive features.
A parent group is always super set of their children groups.

\item {\bf $\ell_1$  + overlap groups}:
$\mathcal{G}$ contains overlapping groups, in which 
each group contains 20 features and consecutive groups overlap by 5 features.

\item {\bf $\ell_1$  + overlap groups guided by prior knowledge}:
This group structure is constructed only 
for AD dataset. $\mathcal{G}$ contains 26,222 
groups, in which 
each group contains
single nucleotide polymorphisms (i.e., features in  AD dataset) located on the same gene
region, defined by the interval between the gene's transcription start and end site.
Note that groups may overlap due to overlapping genes. 

\end{enumerate}

We solved the overlapping group lasso problems with 
OLS, SOLS, and GDPP
on a sequence of 
$\{0.9^1  \lambda', 0.9^2 \lambda', \ldots, 0.9^{30} \lambda'\}$,
where $\lambda'$ was obtained by Algorithm \ref{alg:lambdamax}.
In the sequential screening, we stop using screening from 
$\lambda_{t+1}$ 
if 
no features are  discarded at $\lambda_t$ ($1 \leq t \leq 29$)
because it is likely that 
screening with $\lambda < \lambda_t$   
discards a few features, if at all.
All the experiments except GISETTE (which contains  a single $\by$)  were repeated 
10 times with a different $\by$ for each run, 
and we report the average performance. 
As an overlapping group lasso solver, we used FoGLasso \cite[]{yuan2011efficient},
a state-of-the-art solver in the SLEP 
package \cite[]{liu2009slep}, and all screening rules were implemented in Matlab.
Below, we our demonstrate empirical results under different datasets,
different group structures, and different group sizes to study 
how the screening efficiency changes under various scenarios.

\subsection{Different Datasets}
\label{subsec:screening_efficiency}
\begin{figure}[t!]
\begin{center}
\includegraphics[height=2.3in]{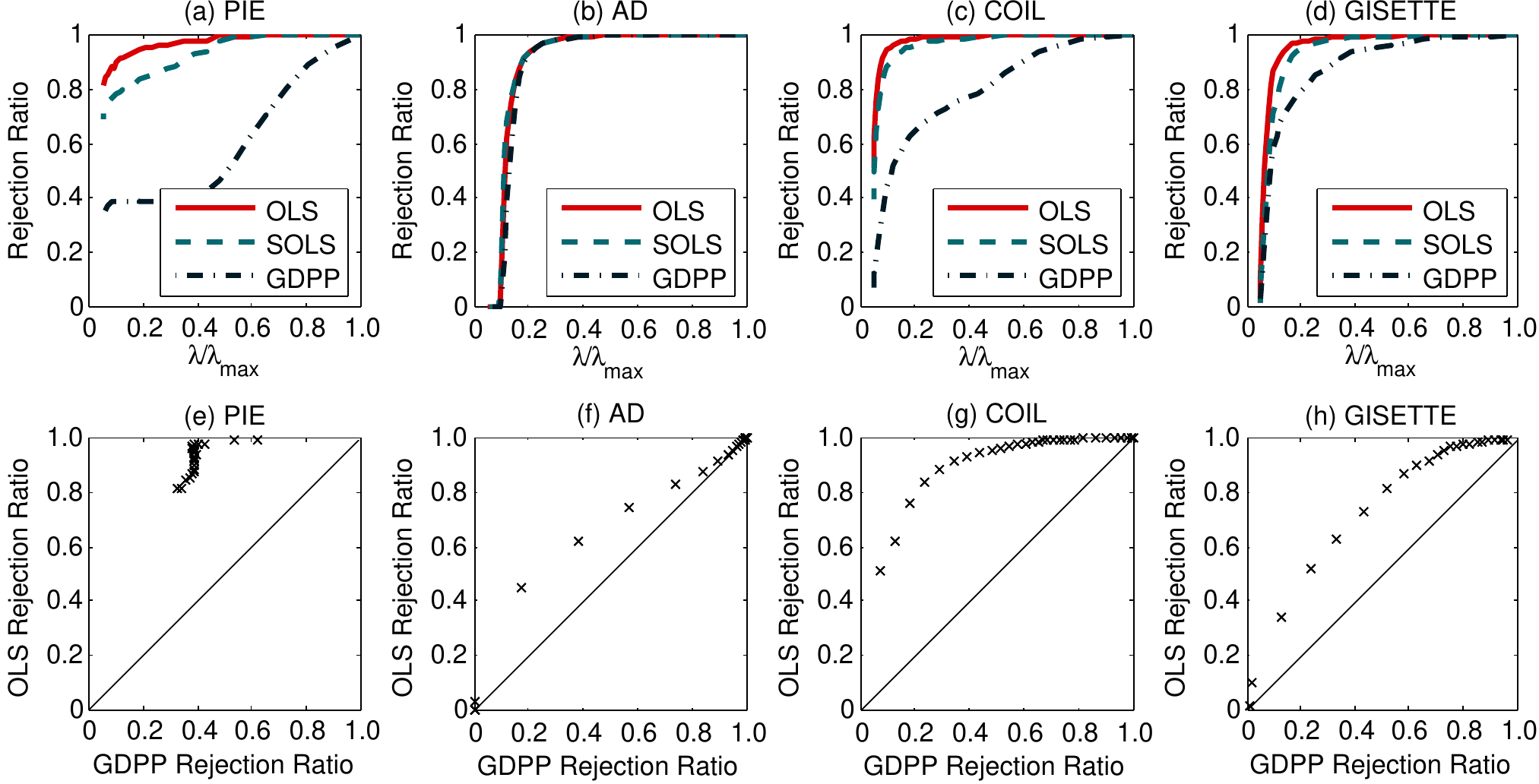} 
\caption{
Rejection ratio of OLS (overlapping group lasso screening), SOLS (sparse overlapping group lasso screening) and GDPP (group dual polytope projection) and their comparison on  (a,e) PIE, (b,f) COIL, (c,g)
AD, 
and (d,h) GISETTE datasets for different $\lambda/ \lambda_{max}$ parameters.} 
\label{fig:lambda_seq} 
\end{center}
\end{figure}

\begin{figure}[t]
\begin{center} 
\includegraphics[height=2.3in]{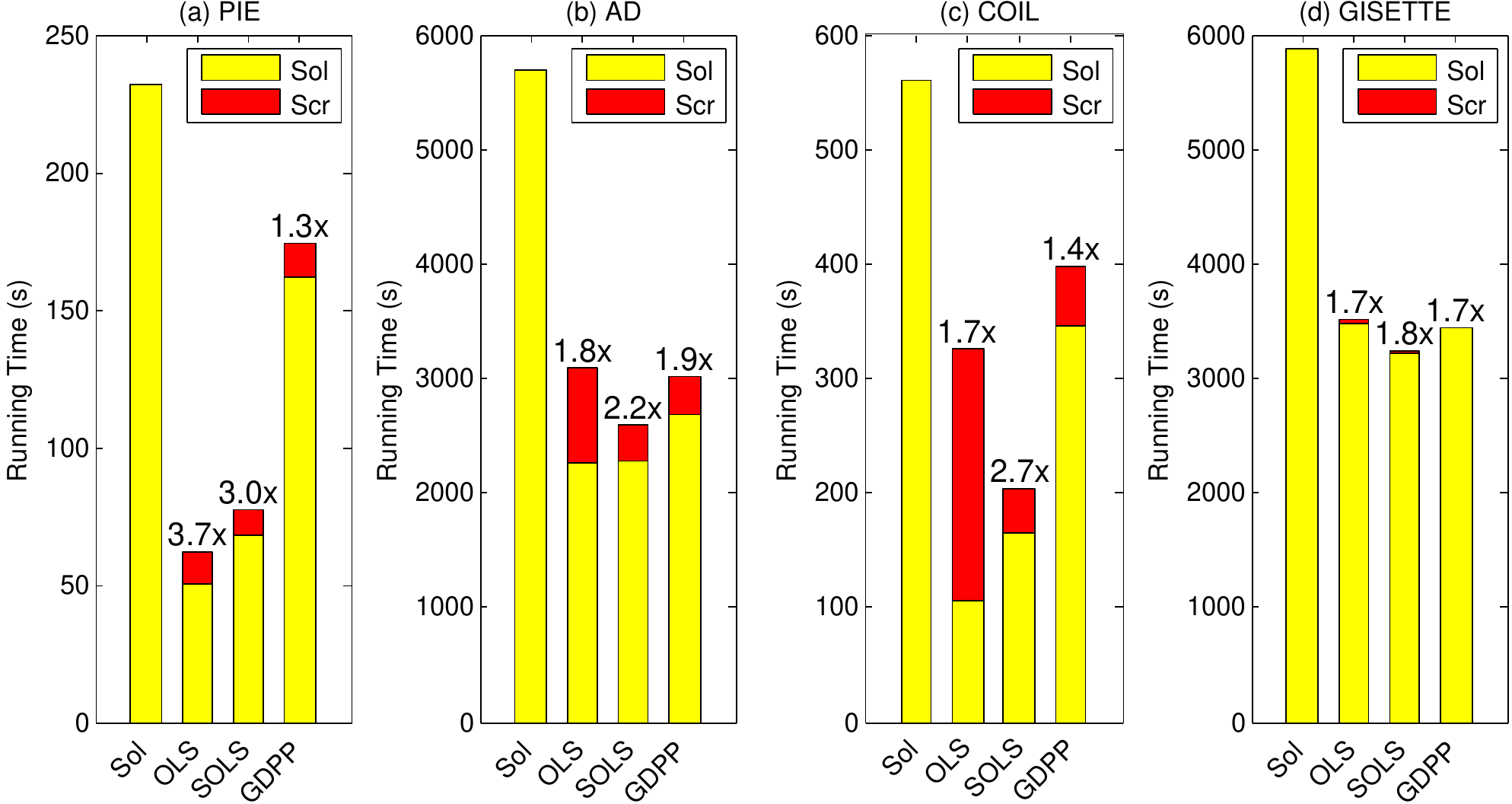} 
\caption{Running times of overlapping group lasso solver
without screening (Sol), solver with OSL screening (OSL),
solver with SOSL screening (SOSL), 
and solver
with GDPP screening (GDPP) on (a) 
PIE, (b) AD, (c) COIL, and (d) GISETTE datasets. } 
\vspace{-.4cm}
\label{fig:lambda_seq_speed} 
\end{center}
\end{figure}

We first investigate the effectiveness of screening algorithms
in terms of 
the screening rejection ratio under the four different datasets.
For this experiment, we used the 
$\ell_1$  + tree structure groups.
Figures \ref{fig:lambda_seq} (a--d) show the rejection ratio  of OLS, SOLS, 
and GDPP on the four datasets as we change $\lambda$ parameter. 
Figures \ref{fig:lambda_seq} (e--h)
 demonstrate 
the differences between screening rejection ratios of the
OLS and those of GDPP; each dot 
represents a result for a single $\lambda$, and the distance 
between dots and the
diagonal line denotes the increased rejection ratio by
OLS's use of overlapping groups.

For all datasets except AD dataset, OLS and SOLS  reject significantly more features than GDPP,
taking advantage of overlapping groups to decrease the left-hand side of 
screening rules. However, given the AD dataset,
the performance gap between OLS (or SOLS) and GDPP 
was the smallest. This phenomenon is due to the small groups
used for the AD dataset. The median of AD group sizes was 8; 
in contrast,
the other datasets used the groups, sizes of 20/15/10. 
In the experiments with different group sizes in $\S$\ref{subset:exp_group_size},
we confirm that
the performance gap between OLS (or SOLS) and GDPP increases as the group size increases.
This is not surprising because as group size increases, the number of free variables 
to minimize the left-hand side of screening rules increases, resulting in a better screening rejection 
ratio. 
Note that OLS and SOLS outperform GDPP for all $N>>J$ (PIE), $N<<J$ (COIL),
and $N\approx J$ (GISETTE) settings. This observation suggests to us that 
OLS and SOLS would be useful for all large datasets regardless of their sample or
feature sizes. 
Furthermore,  OLS and SOLS show similar rejection ratios for all datasets.
Therefore, in a single-machine setting, if $\ell_1$ norm is included in the regularization, SOLS would be preferred over OLS
due to SOLS's low computational complexity as well as its rejection ratio  comparable to that of OLS.

We then measured the speed gain achieved by the screening rules 
on the four different datasets.
In Figure \ref{fig:lambda_seq_speed}, we compare the running
times of solving overlapping group lasso problems given a sequence of $\lambda$
parameters
among the solver without screening, the solver with OLS, the solver with SOLS, 
and the solver with OGDPP.
For the results with screening, we also illustrate the 
portion of running
times consumed by the solver and by screening.
For the PIE, AD, COIL, and GESETTE datasets, we observed $3\times$,
$2.2\times$,  $2.7\times$, and $1.8\times$ speed-up by SOLS
in comparison to the solver without screening.
One interesting observation is that 
for $N>>J$ setting such as in the PIE dataset, running times of screening with
OLS and SOLS are similar; but 
for large $J$ such as in the COIL dataset, OLS is significantly more expensive 
than SOLS. As a result, the solver with OLS was slower than 
the solver with SOLS, even though OLS discarded more features than 
SOLS. 
However, we note that screening rules are embarrassingly parallel.
Thus, in a distributed setting, OLS would be appealing because
the screening portion of running times can be reduced 
proportionally to the number of cores.
Hereafter, we show only the screening rejection ratio by 
different screening algorithms because
the similar patterns of running times are observed under the other 
experimental settings.

\subsection{Different Group Structures}

\begin{figure}[tb]
\begin{center} 
\includegraphics[height=2.3in]{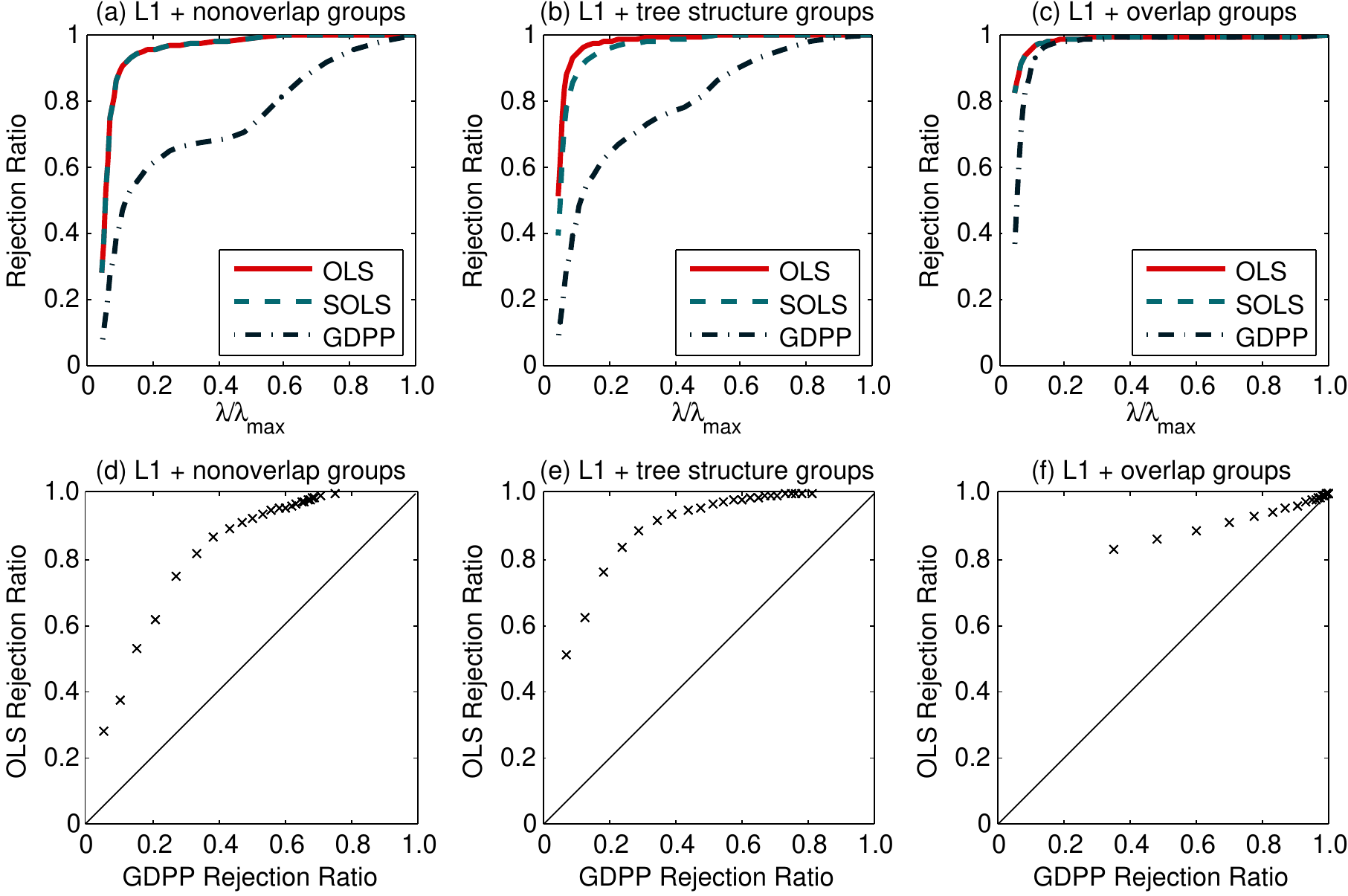} 
\caption{Rejection ratio of OLS, SOLS, and GDPP with different group structures:
(a,d) $\ell_1$ + nonoverlap groups, (b,e) $\ell_1$ + tree structure groups, and 
(c,f) $\ell_1$ + overlap groups.} 
\vspace{-.4cm}
\label{fig:lambda_seq_struct} 
\end{center}
\end{figure}

Next, we run the screening algorithms on the COIL dataset 
under three different group structures, including
$\ell_1$ + nonoverlap groups, $\ell_1$ + tree structure groups, and
$\ell_1$ + overlap groups.
Figures \ref{fig:lambda_seq_struct}(a--c) show the rejection ratios of
different algorithms 
as $\lambda$ changes given different group structures; we compare
the rejection ratio of SOLS with that of GDPP 
below each corresponding plot.
Overall, we can see that OLS and SOLS maintain high rejection ratios
over a wide range of $\lambda$ parameters, but 
GDPP's rejection ratio drops quickly as $\lambda$ decreases.
The rejection ratios by OLS and SOLS
were identical under 
$\ell_1$ + nonoverlap groups and 
$\ell_1$ + overlap groups because
only individual coefficients 
due to $\ell_1$ penalty are inclusive groups for 
both  screening methods.
Interestingly, even under $\ell_1$ + tree structure groups,
OLS and SOLS show similar rejection ratios, which indicates that 
use of  $\ell_1$ penalty for $\bar{\mathcal{G}}_1$ is  effective.

\subsection{Different Group Sizes}
\label{subset:exp_group_size}

\begin{figure}[htb]
\begin{center}\vspace{-.2cm}
\includegraphics[height=2.3in]{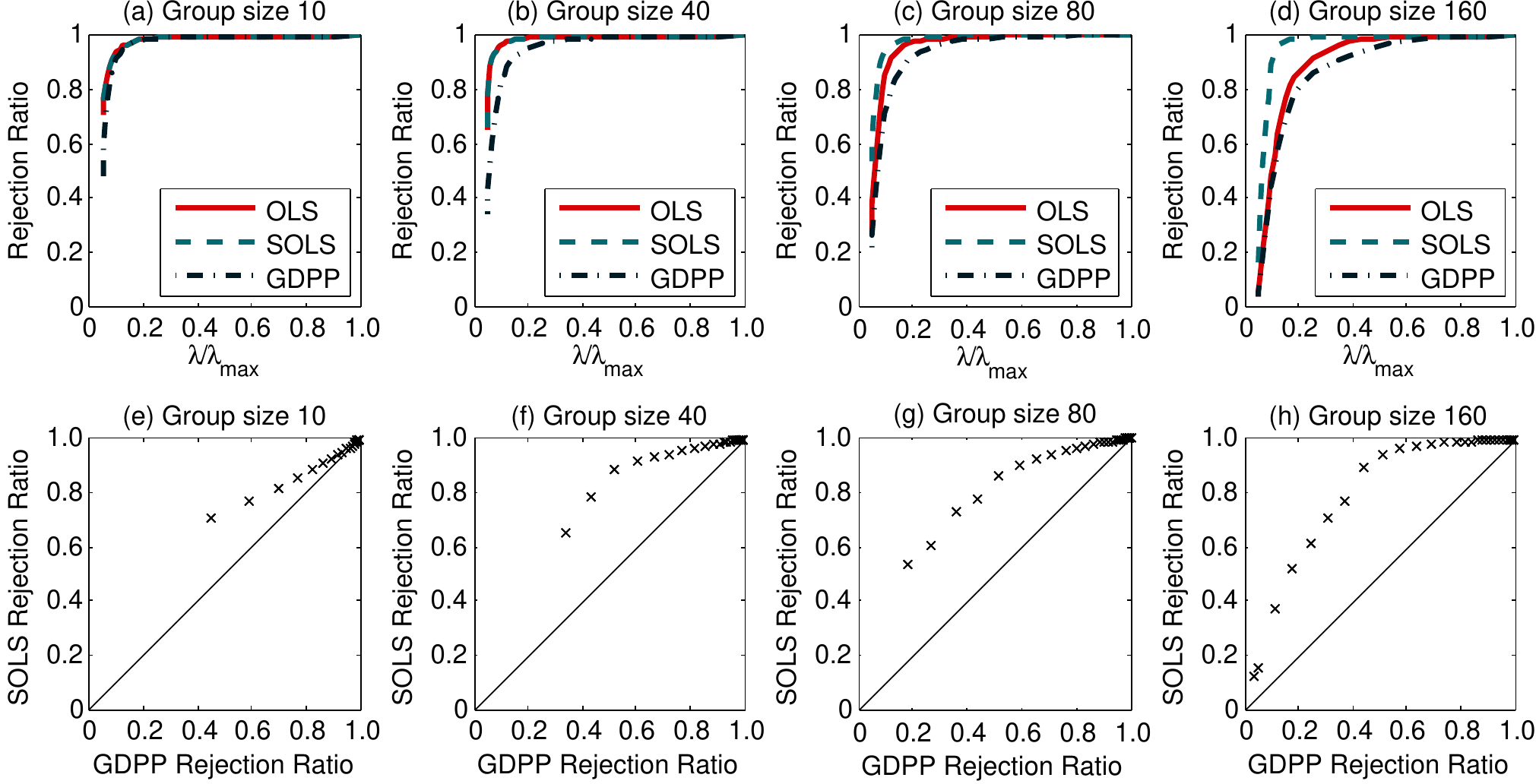} 
\caption{Rejection ratio of OGDPP and GPP with different sizes of overlapping groups.
The sizes of groups are (a) 10 , (b) 40, (c) 80, and (d) 160 features.} 
\label{fig:lambda_grsize} 
\end{center}
\end{figure}

We also observed the effects of group sizes on the screening rejection ratio.
For this experiment, we used the COIL dataset and $\ell_1$ + overlap groups, 
where the size of groups is changed from 10 to 160.
Figure \ref{fig:lambda_grsize} shows the screening efficiency 
of GDPP and OGDPP given different group sizes. 
Clearly, the rejection ratio of GDPP decreases
as the group size increases. 
However, SOLS was not affected by the increased group sizes
because SOLS 
can take advantage of the increased number of 
overlapping groups, $\bar{\mathcal{G}}_1 = \{j : j \in \bg \}$, within each tested group $\bg$.
As the number of overlapping groups increases, 
the number of
free variables to optimize in the left-hand side 
of SOLS rule 
also increases. 
Therefore, SOLS kept the high rejection ratios given all group sizes.
The rejection ratio of OLS is decreased for group sizes $\geq 80$
because we fixed the window size to search for $\bar{\mathcal{G}}_1$ by 50.
Thus, OLS's rejection ratio started to decline from the group size of 80
due to the fixed  number of
overlapping groups in $\bar{\mathcal{G}}_1$.

\section{Conclusion}
In this paper, we developed screening rules including
OLS and SOLS for overlapping group lasso.
We make it possible to screen each group independently of each other
by considering only groups  inclusive of each tested group.
Taking advantage of groups
 that overlap with tested groups,
we showed that OLS and SOLS 
are efficient in terms of the screening rejection ratio.
In addition, we verify that the  
GDPP screening rule  \cite[]{wang2013lasso} 
is a special case of OLS, but it is less efficient than OLS because it lacks
the capability to use overlapping groups.
In OLS, there is a step to find all groups overlapped with the group 
of interest, which can be computationally heavy. Thus, 
as a special case of OLS, we also present SOLS for
sparse overlapping group lasso that includes $\ell_1$ penalty. 
In our experiments, SOLS was much faster than OLS for screening, 
while maintaining its rejection ratio comparable to that of OLS. 
Motivated by enhanced DPP \cite[]{wang2013lasso}, which uses
a tight range of dual optimal solutions, developing more efficient OLS or SOLS would be an interesting research
direction. Furthermore, extending OLS to various loss functions such 
as logistic loss or hinge loss is left for future research.


\appendix

\section{Dual Formulation of Overlapping Group Lasso}

Overlapping group lasso problem is defined by
\begin{equation}
\label{equ:general_supp}
\min_{\bbeta} \frac{1}{2}
\lVert \by - \bX\bbeta  \rVert_2^2 
+ \lambda \sum_{\bg\in \mathcal{G}} \sqrt{n_\bg} \left\| \bbeta_{\bg} \right\|_2, 
\end{equation}
where $\bX \in \mathbb{R}^{N \times J}$ is the input data for  $J$ inputs and $N$ samples, 
$\by \in \mathbb{R}^{N \times 1}$ is the output vector, 
$\bbeta \in \mathbb{R}^{J \times 1}$ is the vector of regression coefficients,
$n_\bg$ is the size of group $\bg$, and 
$\lambda$ is a regularization parameter that determines the sparsity of $\bbeta$.
Here, we convert the primal overlapping group lasso problem in \eqref{equ:general_supp}  to a dual problem.

Introducing $\bz = \by - \bX \bbeta$,  \eqref{equ:general_supp} can be written as
\begin{align}
\label{eq:overlap_g_lasso2_supp}
\min_{\bbeta} & \frac{1}{2} \left\| \bz  \right\|_2^2 + 
\lambda \sum_{\bg\in \mathcal{G}}\sqrt{n_{\bg}} \left\| \bbeta_{\bg} \right\|_2 \\
\mbox{subject to } & \bz = \by - \bX \bbeta. \nonumber
\end{align}

Lagrangian of \eqref{eq:overlap_g_lasso2_supp} is 
\begin{equation}
L(\bbeta,\bz,\etab) =  \frac{1}{2} \left\| \bz  \right\|_2^2 + \lambda  \sum_{g\in \mathcal{G}} 
\sqrt{n_{\bg}} \left\| \bbeta_g \right\|_2 
+\mathbf{\etab}^T \left( \by - \bX \bbeta - \bz \right).
\label{eq:overlap_g_lasso3_supp}
\end{equation}

Dual function $g(\etab)$ of \eqref{eq:overlap_g_lasso3_supp} is
{\footnotesize
\begin{equation}
\begin{split}
g(\etab) &= \textrm{inf}_{\bbeta, \bz} L(\bbeta,\bz,\etab) \\
& = \etab^T \by 
+ \textrm{inf}_{\bbeta} \left( - \etab^T \bX   \bbeta + \lambda \sum_{g\in \mathcal{G}} \sqrt{n_{\bg}} \left\| \bbeta_g \right\|_2  \right)
+ \textrm{inf}_{\bz} \left( \frac{1}{2} \left\| \bz \right\|_2^2 - \etab^T \bz  \right).
\label{eq:overlap_dual}
\end{split}
\end{equation}
}

To obtain $g(\etab)$, 
we solve the following two optimization problems:
\begin{equation}
\label{eq:dual_first}
\textrm{inf}_{\bbeta} \left( - \etab^T  \bX  \bbeta + \lambda \sum_{g\in \mathcal{G}} \sqrt{n_{\bg}} \left\| \bbeta_g \right\|_2  \right)
\end{equation}
and
\begin{equation}
\label{eq:dual_second}
\textrm{inf}_{\bz} \left( \frac{1}{2} \left\| \bz \right\|_2^2 - \etab^T \bz  \right).
\end{equation}

We first solve \eqref{eq:dual_first}. Let us denote
$f_1(\bbeta) \equiv \left( - \etab^T \bX     \bbeta  + \lambda \sum_{g\in \mathcal{G}} \sqrt{n_{\bg}}  \left\| \bbeta_g \right\|_2  \right)$.
A subgradient of $f_1(\bbeta)$ with respect to $\bbeta$ is
\begin{equation}
 \frac{\partial f_1(\bbeta)}{\partial \bbeta} = - \bX^T \etab  + \lambda  \bv, 
\end{equation}
where $\bv$ is a subgradient of 
$\sum_{\bg\in \mathcal{G}} \sqrt{n_\bg} \left\| \bbeta_{\bg} \right\|_2$ 
with respect to $\bbeta$. 
The $j$-th element of $\bv$ 
is given by
\begin{equation}
\label{eq:subgrad_supp}
v_j = 
 \sum_{\{\bg: j \in \bg, \bbeta_\bg \neq {\bf 0}, \bg\in \mathcal{G}\}} 
   \frac{\sqrt{n_{\bg}}\beta_{j}}{\left\| \bbeta_{g} \right\|_2}  +
\sum_{\{\bg: j \in \bg, \bbeta_\bg = {\bf 0}, \bg\in \mathcal{G}\}} 
    \sqrt{n_{\bg}} o_j,
\end{equation}
where 
$o_j$ is a subgradient that satisfies $\left\| \bo_\bg \right\|_2 \leq 1$, where
$j \in \bg$.
To obtain an optimal $\bbeta^{*}$ for \eqref{eq:dual_first}, 
we set  $\frac{\partial f_1(\bbeta)}{\partial \bbeta} = 0$. Then, we have
\begin{equation}
\label{eq:constraint_supp}
\bX^T \etab  = \lambda  \bv. 
\end{equation}

Plugging  
it into $f_1(\bbeta)$, we get
{\scriptsize 
\begin{align}
\label{eq:f1_derive}
f_1(\bbeta) 
&= -\lambda \bv^T\bbeta 
+ \lambda \sum_{\bg\in \mathcal{G}} \sqrt{n_\bg} \left\| \bbeta_\bg \right\|_2  \\
&= -\lambda 
\sum_j \sum_{\{\bg: j \in \bg, \bbeta_\bg \neq {\bf 0}, \bg\in \mathcal{G}\}}
v_j \beta_j
+ \lambda \sum_{\{\bg: \bbeta_\bg \neq \mathbf{0}, \bg\in \mathcal{G}\}}  \sqrt{n_\bg}  \left\| \bbeta_\bg \right\|_2  \\
&=    -\lambda 
\sum_j \sum_{\{\bg: j \in \bg, \bbeta_\bg \neq {\bf 0}, \bg\in \mathcal{G}\}}
     \frac{\sqrt{n_\bg}(\beta_{j})^2}{\left\| \bbeta_{\bg} \right\|_2}  
+ \lambda \sum_{\{\bg: \bbeta_\bg \neq \mathbf{0}, \bg\in \mathcal{G}\}}  \sqrt{n_\bg}  \left\| \bbeta_\bg \right\|_2  \label{eq:f1_derive_3}\\
&=    -\lambda  \sum_{\{\bg: \bbeta_\bg \neq 0, \bg\in \mathcal{G}\}} \sum_{j \in \bg}
     \frac{\sqrt{n_\bg}(\beta_{j})^2}{\left\| \bbeta_{\bg} \right\|_2}  
+ \lambda \sum_{\{\bg: \bbeta_\bg \neq \mathbf{0}, \bg\in \mathcal{G}\}}  \sqrt{n_\bg}  \left\| \bbeta_\bg \right\|_2   \\
&= 0. 
\end{align}
}
Here \eqref{eq:subgrad} is used for \eqref{eq:f1_derive_3}.
Therefore, $\textrm{inf}_{\bbeta} f_1(\bbeta) = 0$.

Now we solve the second problem $f_2(\bz) \equiv   \frac{1}{2} \left\| \bz \right\|_2^2 - \etab^T \bz   $.
This result was presented in \cite[]{wang2013lasso}. However, here we show the derivation for self-containedness. 
\begin{align}
f_2(\bz) =  \frac{1}{2} \left\| \bz \right\|_2^2 - \etab^T \bz   =
 \frac{1}{2} \left( \left\| \bz -\etab \right\|_2^2 - \left\| \etab \right\|_2^2  \right). 
\end{align}
Note that  $\etab  = \argmin_{\bz}  \frac{1}{2} \left( \left\| \bz -\etab \right\|_2^2 - \left\| \etab \right\|_2^2  \right)$,
and thus $\textrm{inf}_{\bz} f_2(\bz) = -\frac{1}{2} \left\| \etab \right\|_2^2$.

Based on these results for \eqref{eq:dual_first}
and \eqref{eq:dual_second}, the dual function $g(\etab)$ 
in \eqref{eq:overlap_dual} is given by,
\begin{equation}
\label{eq:dual_g}
g(\etab) = \etab^T \by - \frac{1}{2} \left\| \etab \right\|_2^2 = 
\frac{1}{2} \left\| \by \right\|_2^2  -\frac{1}{2} \left\| \etab - \by \right\|_2^2.
\end{equation}
Finally, we denote $\thetab = \frac{\etab}{\lambda}$.
Combining \eqref{eq:dual_g} with \eqref{eq:constraint_supp}, 
a dual formulation of overlapping group lasso is as follows:
\begin{align}
\label{eq:dual_formul_supp}
& \textrm{sup}_{\thetab}  \frac{1}{2} \left\| \by \right\|_2^2  -\frac{\lambda^2}{2} \left\| \thetab - \frac{\by}{\lambda} \right\|_2^2 \\
& \mbox{subject to }   \bX^T \thetab   =   \bv. \nonumber
\end{align}

\section{Proof of Theorem 1}
\label{appendix:proof_theorem1}
\setcounter{theorem}{0}
\begin{theorem}
For the overlapping lasso problem, suppose that we are given an optimal 
dual solution
$\thetab^{*}(\lambda_0)$. 
Then for $\lambda < \lambda_0$, $\bbeta_\bg^*(\lambda) = \mathbf{0}$ if
{\scriptsize  
\begin{equation}
\label{eq:theorem_screen_supp}
 \min_{\bw_\bh, \left\| \bw_\bh \right\|_2 \leq 1}   
\sqrt{ \sum_{j \in \bg}
 \left(
\bx_j^T \thetab^*(\lambda_0)    
   - \sum_{j\in \bh, \bh \in \bar{\mathcal{G}}_1} \sqrt{n_\bh} w_j  
   \right)^2}
   < \sqrt{n_\bg} -  \left\| \bX_\bg \right\|_F \left\| \by \right\|_2  \left| \frac{1}{\lambda} - \frac{1}{\lambda_0} \right|.
\end{equation}
}
\label{theorem:screen_supp}
\end{theorem}
\begin{proof}
Based on \eqref{eq:nonex}, 
we have a sphere $\Thetab$
that contains $\thetab^*(\lambda)$, which is 
centered at 
$\thetab^*(\lambda_0)$ with a radius of 
$\rho = \left\| \frac{\by}{\lambda} - \frac{\by}{\lambda_0} \right\|_2$.
Thus, we can represent $\thetab^*(\lambda) = \thetab^*(\lambda_0) + \br$,
where $\left\| \br \right\|_2 \leq  \rho$.
Plugging it into \eqref{eq:mini_opt_problem} we get
{\scriptsize 
\begin{align}
b_\bg 
&\leq   \min_{\bw_\bh: \left\| \bw_\bh \right\|_2 \leq 1}  
\sqrt{\sum_{j \in \bg} 
\left(  \bx_j^T \thetab^*(\lambda_0) + \bx_j^T\br
   - \sum_{j\in \bh, \bh \in \bar{\mathcal{G}}_1} \sqrt{n_\bh} w_j  \right)^2} 
     \label{eq:screen2}\\
 &  \leq    \min_{\bw_\bh: \left\| \bw_\bh \right\|_2 \leq 1} 
 \sqrt{ \sum_{j \in \bg} 
\left(  \bx_j^T \thetab^*(\lambda_0)
   - \sum_{j\in \bh, \bh \in \bar{\mathcal{G}}_1} \sqrt{n_\bh} w_j  
   \right)^2} + 
 \sqrt{\sum_{j \in \bg}    \left(\bx_j^T \br \right)^2} 
     \label{eq:screen4}\\
&   \leq    \min_{\bw_\bh: \left\| \bw_\bh \right\|_2 \leq 1}   
\sqrt{\sum_{j \in \bg} \left(
 \bx_j^T \thetab^*(\lambda_0)
   - \sum_{j\in \bh, \bh \in \bar{\mathcal{G}}_1} \sqrt{n_\bh} w_j  
  \right)^2} +  
  \sqrt{\left\| \bX_\bg \right\|_F^2 \left\|\by\right\|_2^2 
  \left( \frac{1}{\lambda} - \frac{1}{\lambda_0} \right)^2}
  \label{eq:screen5}.
\end{align}
}
We used 
Minkowski's inequality for \eqref{eq:screen4}, and
Cauchy-Schwarz inequality and
$\left\| \bX_\bg \right\|_2^2 \leq \left\| \bX_\bg \right\|_F^2$ for \eqref{eq:screen5}.
From \eqref{eq:mini_opt_problem}, if $b_\bg< \sqrt{n_\bg}$,
then $\bbeta_\bg^* = {\bf 0}$, and the result follows.
\end{proof}

\end{document}